\documentclass{article}
\usepackage{times}
\usepackage{epsfig}
\usepackage{graphicx}
\usepackage{amsmath, amsthm, amssymb}
\usepackage{algorithm}
\usepackage{algorithmic}
\usepackage{natbib}
\usepackage{hyperref}
\usepackage{fullpage}
\usepackage{pgfplots}
\usepackage{nicefrac}
\usepackage{enumerate}
\pgfplotsset{compat=1.11}
\usetikzlibrary{arrows.meta}

\newcommand{\bg}{\boldsymbol{g}}

\newcommand{\bq}{\boldsymbol{q}}
\newcommand{\bx}{\boldsymbol{x}}
\newcommand{\bu}{\boldsymbol{u}}
\newcommand{\by}{\boldsymbol{y}}

\newcommand{\bv}{\boldsymbol{v}}

\newcommand{\argmin}{\mathop{\mathrm{argmin}}}

\newcommand{\dom}{\mathop{\mathrm{dom}}}

\newcommand{\field}[1]{\mathbb{#1}}

\newcommand{\R}{\field{R}}

\newtheorem{theorem}{Theorem}
\newtheorem{lemma}{Lemma}
\newtheorem{cor}{Corollary}
\newtheorem{remark}{Remark}

\newtheorem{definition}{Definition}

\begin{document}

\title{Normalized Gradients for All}
\author{Francesco Orabona\thanks{Work done mainly while at Boston University.}\\
KAUST\\
{\tt\small francesco@orabona.com}
}

\maketitle

\begin{abstract}
In this short note, I show how to adapt to H\"{o}lder smoothness using normalized gradients in a black-box way. Moreover, the bound will depend on a novel notion of local H\"{o}lder smoothness.
The main idea directly comes from \citet{Levy17}.
\end{abstract}

\section{Introduction}

Here, I present\footnote{A preliminary version of this report appeared on June 19th 2023 as a blogpost: \url{https://parameterfree.com/2023/06/19/adapting-to-smoothness-with-normalized-gradients/}.} some results about the use of normalized gradients to obtain adaptivity to H\"{o}lder smoothness.
The reduction is very generic, so I will apply it to OGD, Dual Averaging, and parameter-free algorithms.

First, I will consider a close relative of normalized gradients: AdaGrad-norm stepsizes. Then, I will show similar results for the normalized gradients.

The core ideas are directly derived from \citet{Levy17}. Indeed, the main aim of this note is to show how some very recent optimization results on normalized gradients are in fact well-known in the online learning community. The hope of the author is to instill a major awareness and academic respect for online learning results in the optimization community.

\section{Definitions and Setup}
In this section, I describe the definitions and setup.

\paragraph{H\"{o}lder Smoothness}
For the following definition, see \citet{Nesterov15b}.
\begin{definition}
Let $\nu \in [0,1]$ and define the class of $(L_\nu,\nu)$-H\"{o}lder smooth functions with respect to $\|\cdot\|$ the ones that satisfy
\[
L_\nu = \max_{\bx,\by \in \R^d, \bx\neq \by} \frac{\|\nabla f(\bx) - \nabla f(\by)\|_\star}{\|\bx - \by\|^\nu} < \infty
\]
where $\|\cdot\|_\star$ is the dual norm of $\|\cdot\|$.
\end{definition}
This definition is useful because it allows us to capture a wide range of functions in a unified way.
It is a generalization of smoothness and Lipschitzness, in fact for smooth functions $\nu=1$ and for Lipschitz functions $\nu=0$.

For this class of functions, we can prove the following well-known result \citep[see, e.g.,][]{Nesterov15b}.
\begin{theorem}
\label{thm:holder}
Let $f:\R^d \to \R$ be a $(L_\nu,\nu)$-H\"{o}lder smooth function with respect to $\|\cdot\|$.
Then, for $\bx, \by \in \dom f$ and $\bg \in \nabla f(x)$, we have 
\[
f(\by) \leq  f(\bx) + \langle \nabla f(\bx), \by-\bx\rangle + \frac{L_\nu}{1+\nu}\|\bx-\by\|^{1+\nu}~.
\]
\end{theorem}
\begin{proof}
\begin{align*}
f(\by) 
&= f(\bx) + \int_0^1 \langle \nabla f(\bx+\tau (\by-\bx)), \by -\bx\rangle d \tau \\
&= f(\bx) + \langle \nabla f(\bx), \by-\bx\rangle + \int_0^1 \langle \nabla f(\bx+\tau (\by-\bx))- \nabla f(\bx), \by -\bx\rangle d \tau~.
\end{align*}
Therefore,
\begin{align*}
|f(\by)-f(\bx) -\langle \nabla f(\bx), \by -\bx\rangle|
&= \left|\int_0^1 \langle \nabla f(\bx+\tau (\by-\bx))- \nabla f(\bx), \by -\bx\rangle d \tau\right| \\
&\leq \int_0^1 |\langle \nabla f(\bx+\tau (\by-\bx))- \nabla f(\bx), \by -\bx\rangle| d \tau \\
&\leq \int_0^1 \|\nabla f(\bx+\tau (\by-\bx))- \nabla f(\bx)\|_\star \|\by -\bx\| d \tau \\
&\leq \int_0^1 \tau^\nu L_\nu \|\by -\bx\|^{1+\nu} d \tau
= \frac{L_\nu}{1+\nu} \|\by -\bx\|^{1+\nu}~. \qedhere
\end{align*}
\end{proof}

The following Corollary is an immediate consequence.
\begin{cor}
\label{cor:holder}
Let $\nu \in(0,1]$, $f:\R^d \to \R$ be $(L_\nu,\nu)$-H\"{o}lder smooth with respect to $\|\cdot\|$, and $\bx^\star \in \argmin_{\bx} f(\bx)$.
Then, for any $\bx \in \R^d$
\[
\|\nabla f(\bx)\|_\star^{\frac{1}{\nu}+1}
\leq \left(1+\frac{1}{\nu}\right) (L_\nu)^\frac{1}{\nu}(f(\bx) -  f(\bx^\star))~.
\]
\end{cor}
\begin{proof}
Let $\bv$ is any vector such that $\langle \nabla f(\bx), \bv\rangle=-\|\nabla f(\bx)\|_\star \|\bv\|$ and scaled in such a way that $\|\bv\| = \alpha \|\nabla f(\bx)\|_\star $. Then, from Theorem~\ref{thm:holder} with $\by=\bx+\bv$, we have
\[
f(\bx^\star)
\leq f(\bx+\bv) 
\leq  f(\bx) - \alpha \|\nabla f(\bx)\|_\star^2 + \alpha^{1+\nu} \frac{L_\nu}{1+\nu}\|\nabla f(\bx)\|_\star^{1+\nu}~.
\]
So, for any $\alpha>0$, we have
\[
f(\bx^\star) - f(\bx) 
\leq -\alpha \|\nabla f(\bx)\|_\star^2 + \alpha^{1+\nu} \frac{L_\nu}{1+\nu}\|\nabla f(\bx)\|_\star^{1+\nu}~.
\]
With the optimal setting of $\alpha= (L_\nu)^{-\frac{1}{\nu}} \|\nabla f(\bx)\|^\frac{1-\nu}{\nu}$, the right hand side of this inequality becomes
\begin{align*}
&-(L_\nu)^{-\frac{1}{\nu}} \|\nabla f(\bx)\|_\star^\frac{1-\nu}{\nu} \|\nabla f(\bx)\|_\star^2 + (L_\nu)^{-\frac{1+\nu}{\nu}} \|\nabla f(\bx)\|_\star^\frac{(1+\nu)(1-\nu)}{\nu} \frac{L_\nu}{1+\nu}\|\nabla f(\bx)\|_\star^{1+\nu} \\
&=-(L_\nu)^{-\frac{1}{\nu}} \|\nabla f(\bx)\|_\star^{\frac{1}{\nu}+1}+\frac{1}{1+\nu}(L_\nu)^{-\frac{1}{\nu}}\|\nabla f(\bx)\|_\star^{\frac{1}{\nu}+1}\\
&=-\frac{\nu}{1+\nu}(L_\nu)^{-\frac{1}{\nu}} \|\nabla f(\bx)\|_\star^{\frac{1}{\nu}+1}~.
\end{align*}
Reordering, we have the stated bound.
\end{proof}

\paragraph{Setup} Our objective is to minimize a convex H\"{o}lder smooth function $f:\R^d \to \R$. However, we do not know $\nu$ and $L_\nu$. We have access to the gradient $\nabla f(\bx)$ in any query point $\bx$. For shortness of notation, we will denote by $\bg_t:=\nabla f(\bx_t)$ where $\bx_t$ is the point where we query the gradient at iteration $t$.

\section{Warm-up: AdaGrad-Norm Adapts to H\"{o}lder Smoothness}

Consider to use Gradient Descent: $\bx_{t+1} = \bx_t - \eta_t \nabla f(\bx_t)$.
It should be easy to see that the optimal learning rate $\eta_t$ will depend on the H\"{o}lder exponent $\nu$. In fact, it is enough to consider the two extreme cases, $\nu=0$ and $\nu=1$. For $\nu=0$, the function is Lipschitz and the optimal learning rate is $\eta_t \propto \frac{1}{\sqrt{T}}$ or $\eta_t \propto \frac{1}{\sqrt{t}}$ that would give you a convergence rate of $O(\frac{1}{\sqrt{T}})$. Instead, for the case $\nu=1$ we can get the rate\footnote{Note that in the smooth case ($\nu=1$) the rate is not optimal because you can get the accelerated one that is $O(1/T^2)$, for example using Nesterov's momentum algorithm~\citet{Nesterov83}.} $O(1/T)$ using a constant learning rate, independent of $T$. So, you should know in which case you are in order to set the right learning rate.

From this reasoning, you should see that the rate that gradient descent (without acceleration) can achieve on H\"{o}lder smooth functions depend on $\nu$ and it goes from $O(1/\sqrt{T})$ to $O(1/T)$. See, for example, \citet[Corollary 9]{Grimmer19} for the precise learning rate and convergence guarantee.

Now, our objective is simple: We want to achieve the above rate without knowing $\nu$ for all the H\"{o}lder smooth functions. Sometime people call this kind of property ``adaptivity'' or ``universality'', even if universality as defined by \citet{Nesterov15b} would achieve the optimal accelerated rate for all $\nu$.

Let's now warm up a bit considering the so-called AdaGrad-norm stepsizes~\citep{StreeterM10}, that is $\eta_t=\frac{\alpha}{\sqrt{\sum_{i=1}^t \|\bg_i\|_2^2}}$.

It is known that gradient descent with the AdaGrad-norm stepsizes gets a faster rate on smooth functions, that is $O(1/T)$ instead of the usual $O(1/\sqrt{T})$~\citep{LiO19}. So, what happens on H\"{o}lder smooth functions? Given that Nature tends to be continuous, it should be intuitive that AdaGrad-norm stepsizes should give us a rate that smoothly interpolates between the Lipschitz and smooth case. Let's see how.

I will consider the Follow-The-Regularized-Leader (FTRL) with linearized losses~\citep{Gordon99,ShalevS06,AbernethyHR08,HazanK08} (or as it is known in the optimization community Dual Averaging (DA)~\citep{Nesterov09}) version of AdaGrad-norm to have a simpler proof, but the analysis in the gradient descent case is also possible. In this case, the update is
\[
\bx_{t+1} 
=\bx_1 - \frac{\alpha}{\sqrt{G^2+\sum_{i=1}^{t} \|\bg_i\|_2^2}}\sum_{i=1}^t \bg_i,
\]
where $G \geq \|\bg_i\|_2$ for all $i$ and $\bx_1$ is arbitrarily chosen. 
From its known guarantee, we have
\begin{align*}
\sum_{t=1}^T (f(\bx_t) - f(\bx^\star))
\leq \left(\frac{\|\bx_1-\bx^\star\|_2^2}{2\alpha} + \alpha\right) \sqrt{G^2+\sum_{t=1}^T \|\bg_t\|_2^2}
\leq \left(\frac{\|\bx_1-\bx^\star\|_2^2}{2\alpha} + \alpha\right) \left(G+\sqrt{\sum_{t=1}^T \|\bg_t\|_2^2}\right)~.
\end{align*}
Now, we want to transform the terms $\|\bg_t\|_2^2$ in $\|\bg_t\|_2^{1+1/\nu}$ to be able to apply Corollary~\ref{cor:holder}. So, focus on the last term and use H\"{o}lder's inequality:
\begin{align*}
\sum_{t=1}^T \|\bg_t\|^2_2
\leq \left(\sum_{t=1}^T 1^{q}\right)^{1/q} \left(\sum_{t=1}^T \|\bg_t\|_2^{2p}\right)^\frac{1}{p}
= T^{1-\frac{1}{p}}\left(\sum_{t=1}^T \|\bg_t\|_2^{2p}\right)^\frac{1}{p}~.
\end{align*}
Setting $p= \frac{1}{2}+\frac{1}{2\nu}$ so that $1/p=\frac{2\nu}{\nu+1}$, we have
\begin{align*}
\sqrt{\sum_{t=1}^T \|\bg_t\|^2}
\leq T^{\frac{-\nu+1}{2(\nu+1)}}\left(\sum_{t=1}^T \|\bg_t\|_2^{1+1/\nu}\right)^\frac{\nu}{\nu+1}
\end{align*}
If the function $f$ is $(L_\nu,\nu)$-H\"{o}lder smooth, from Corollary~\ref{cor:holder} we have
\begin{align*}
\sum_{t=1}^T (f(\bx_t)-f(\bx^\star)) 
&\leq \left(\frac{\|\bx_1-\bx^\star\|_2^2}{2\alpha} + \alpha\right) \left(G+\sqrt{\sum_{t=1}^T \|\bg_t\|_2^2}\right) \\
&\leq \left(\frac{\|\bx_1-\bx^\star\|_2^2}{2\alpha} + \alpha\right) \left[G+T^\frac{-\nu+1}{2(\nu+1)}\left(\sum_{t=1}^T \|\bg_t\|_2^{1+\frac{1}{\nu}}\right)^\frac{\nu}{\nu+1}\right] \\
&\leq \left(\frac{\|\bx_1-\bx^\star\|_2^2}{2\alpha} + \alpha\right) \left[G+T^\frac{-\nu+1}{2(\nu+1)}\left[(1/\nu+1)L_\nu^{1/\nu}  \sum_{t=1}^T (f(\bx_t)-f(\bx^\star)) \right]^\frac{\nu}{\nu+1}\right]~.
\end{align*}
Using a case analysis on the value of $G$, we have
\begin{align*}
\sum_{t=1}^T (f(\bx_t)-f(\bx^\star)) 
\leq \max\left(L_\nu \left(\frac{1}{\nu}+1\right)^\nu \left(\frac{\|\bx_1-\bx^\star\|_2^2}{\alpha} + 2\alpha\right)^{\nu+1} T^\frac{-\nu+1}{2}, G\left(\frac{\|\bx_1-\bx^\star\|_2^2}{\alpha} + 2\alpha\right)\right)~.
\end{align*}
The online-to-batch conversion completes the convergence rate:
\[
f(\bar{\bx}_T)-f(\bx^\star)
\leq \max\left(L_\nu \left(\frac{1}{\nu}+1\right)^\nu \left(\frac{1}{\sqrt{T}}\left(\frac{\|\bx_1-\bx^\star\|_2^2}{2\alpha} + \alpha\right)\right)^{\nu+1}, \frac{G}{T}\left(\frac{\|\bx_1-\bx^\star\|_2^2}{\alpha} + 2\alpha\right)\right),
\]
where $\bar{\bx}_T$ is the average of the iterates, $\frac1T\sum_{t=1}^T f(\bx_t)$. 

So, we obtained that AdaGrad-norm stepsizes will give a $O(1/\sqrt{T})$ rate for Lipschitz functions, $O(1/T)$ for the usual smooth function and a rate in between for H\"{o}lder smooth functions, without the need to know $L_\nu$ nor $\nu$. 
In the next section, we will see how to generalize this idea a bit more. Also, I will show a way to prove that the iterates are bounded, that can be used on AdaGrad-norm stepsizes too.

\section{Adapting to Local H\"{o}lder Smoothness with Normalized Gradients}

From the previous section, it is unclear what is the key ingredient in AdaGrad-norm stepsizes that gives us the adaptivity. Is the use of the normalization by the past gradients or the normalization by the current gradient would suffice?

It turns out that normalized gradients are enough. Moreover, I will prove a stronger result: The algorithm will adapt to a certain notion of ``local'' H\"{o}lder smoothness.
The proposed method is a black-box reduction: you can use it with any online optimization algorithm that guarantees an upper bound to its regret. This procedure is summarized in Algorithm~\ref{alg:nogd}.

\begin{algorithm}[t]
\caption{Optimization with Normalized Gradients}
\label{alg:nogd}
\begin{algorithmic}[1]
{
    \REQUIRE{An online linear learning algorithm $\mathcal{A}$, a function $f:\R^d \to\R$ to minimize}
    \FOR{$t=1$ {\bfseries to} $T$}
    \STATE{Get $\bx_t$ from $\mathcal{A}$}
    \STATE{Set $\bg_t= \nabla \ell_t(\bx_t)$}
    \IF{$\|\bg_t\|_\star=0$}
    \RETURN{$\bx_t$, optimal point}
    \ENDIF
    \STATE{Pass to $\mathcal{A}$ the loss $\ell_t(\bx)=\left\langle \frac{\bg_t}{\|\bg_t\|_\star}, \bx\right\rangle$}
    \ENDFOR
    \RETURN{$\bar{\bx}_T=\frac{1}{\sum_{t=1}^T \frac{1}{\|\bg_t\|_\star}}\sum_{t=1}^T \frac{\bx_t}{\|\bg_t\|_\star}$}
}
\end{algorithmic}
\end{algorithm}

Let's see the details in the following theorem.
\begin{theorem}
\label{thm:main}
Suppose you have an online linear optimization algorithm $\mathcal{A}$ that guarantees
\[
\sum_{t=1}^T \langle \bq_t, \bx_t - \bu\rangle
\leq \psi_T(\|\bu\|)
\]
when fed with linear losses $\ell_t(\bx)=\langle\bq_t,\bx\rangle$ where $\|\bq_t\|_\star= 1$ for all $t$, for some function $\psi_T:\R^d \to \R$ and any $\bu \in \R^d$.
Let $f:\R^d \to \R$ convex and $\bx^\star \in \argmin_{\bx} f(\bx)$. Assume that there exists $\nu \in [0,1]$, $\alpha>0$, and $L(\bx)$ such that for all $\bx \in \R^d$ we have
\[
\|\nabla f(\bx)\|_\star
\leq \alpha^\frac{\nu}{1+\nu} (L(\bx))^\frac{1}{\nu+1}(f(\bx) -  f(\bx^\star))^\frac{\nu}{1+\nu}~.
\]
Then, for any $\nu \in [0,1]$, Algorithm~\ref{alg:nogd} guarantees
\begin{align*}
f(\bar{\bx}_T) - f(\bx^\star)
\leq \alpha^\nu \left(\frac{\psi_T(\bx^\star)}{T}\right)^{1+\nu} \left( \prod_{t=1}^T L(\bx_t)\right)^\frac{1}{T} 
\leq \alpha^\nu \left(\frac{\psi_T(\bx^\star)}{T}\right)^{1+\nu} \frac{1}{T}\sum_{t=1}^T L(\bx_t) ~.
\end{align*}
\end{theorem}

Here, the adaptive rate depends on the \emph{geometric mean} of the ``local'' smoothnesses on the iterates $L(\bx_t)$, that in turn is upper bounded by their arithmetic mean. Note that if the function is $(L_\nu,\nu)$-H\"{o}lder smooth everywhere, then $\alpha^\nu\leq \left(1+\frac{1}{\nu}\right)^\nu\in[1,2]$ and $L(\bx_t)\leq L_\nu$, improving the result in the previous section.

To prove this theorem, we will use the HM-GM-AM inequality
\begin{lemma}
\label{lemma:hm_am}
Let $a_1, \dots, a_T$ positive numbers. Then, we have
\[
\frac{T}{\sum_{t=1}^T \frac{1}{a_t}}
\leq \left(\prod_{t=1}^T a_t\right)^{1/T}
\leq \left(\frac{1}{T} \sum_{t=1}^T a_t\right)~.
\]
\end{lemma}

We can now prove our main Theorem. The proof is immediate from the proof of the results in \citet{Levy17}.
\begin{proof}[Proof of Theorem~\ref{thm:main}]
From the guarantee of the algorithm and the setting of $\bq_t$ we obtain for all $\bu \in \R^d$
\begin{align*}
\sum_{t=1}^T \frac{f(\bx_t) - f(\bu)}{\|\bg_t\|_\star}
\leq \sum_{t=1}^T \frac{\langle \bg_t, \bx_t -\bu\rangle}{\|\bg_t\|_\star}
\leq \psi_T(\bu)~.
\end{align*}

Now, from the definition of $\bar{\bx}_T$ and Jensen's inequality, we have
\[
\sum_{t=1}^T \frac{f(\bx_t) -f(\bx^\star)}{\|\bg_t\|_\star}
\geq (f(\bar{\bx}_T) -f(\bx^\star))\sum_{t=1}^T \frac{1}{\|\bg_t\|_\star}~.
\]
Hence, we have
\begin{align*}
f(\bar{\bx}_T) - f(\bx^\star)
\leq \frac{\psi_T(\bx^\star)}{\sum_{t=1}^T \frac{1}{\|\bg_t\|_\star} } 
\leq \frac{\psi_T(\bx^\star)}{T} \left(\prod_{t=1}^T \|\bg_t\|_\star\right)^\frac{1}{T}~.
\end{align*}
In the case $\nu=0$, the above inequality gives the stated guarantee.

Let's now assume that $\nu>0$ and, continuing the chain of inequalities, we have
\begin{align*}
\frac{\psi_T(\bx^\star)}{T} \left(\prod_{t=1}^T \|\bg_t\|_\star\right)^\frac{1}{T}
&= \frac{\psi_T(\bx^\star)}{T} \left(\prod_{t=1}^T \frac{\|\bg_t\|_\star^{1+\frac{1}{\nu}}}{\|\bg_t\|_\star}\right)^\frac{\nu}{T} \\
&\leq \frac{\psi_T(\bx^\star)}{T} \left( \prod_{t=1}^T \frac{\alpha(L(\bx_t))^\frac{1}{\nu}(f(\bx_t) -  f(\bx^\star))}{\|\bg_t\|_\star}\right)^\frac{\nu}{T} \\
&= \frac{\psi_T(\bx^\star)}{T}\alpha^\nu \left( \prod_{t=1}^T L(\bx_t)\right)^\frac{1}{T} \left( \prod_{t=1}^T \frac{f(\bx_t) -  f(\bx^\star)}{\|\bg_t\|_\star} \right)^\frac{\nu}{T}\\
&\leq \frac{\psi_T(\bx^\star)}{T} \alpha^\nu \left( \prod_{t=1}^T L(\bx_t)\right)^\frac{1}{T} \left( \frac{1}{T}\sum_{t=1}^T \frac{f(\bx_t) -  f(\bx^\star)}{\|\bg_t\|_\star} \right)^{\nu}\\
&\leq \frac{\psi_T(\bx^\star)}{T} \alpha^\nu \left( \prod_{t=1}^T L(\bx_t)\right)^\frac{1}{T} \left(\frac{\psi_T(\bx^\star)}{T}\right)^{\nu}~.
\end{align*}
where we used Lemma~\ref{lemma:hm_am} in the first inequality, Corollary~\ref{cor:holder} in the second one, the assumption on the online learning algorithm in the last one.
\end{proof}

We now consider some examples.

\paragraph{Online Gradient Descent}
Consider online gradient descent with constant learning rate $\eta=\frac{\alpha}{\sqrt{T}}$. In this case, Algorithm~\ref{alg:nogd} simply becomes Algorithm~\ref{alg:nogd2}.

\begin{algorithm}[t]
\caption{Normalized Gradient Descent}
\label{alg:nogd2}
\begin{algorithmic}[1]
{
    \REQUIRE{A function $f:\R^d \to\R$ to minimize}
    \STATE{Set $\bx_1\in \R^d$ arbitrary}
    \FOR{$t=1$ {\bfseries to} $T$}
    \STATE{Set $\bg_t= \nabla \ell_t(\bx_t)$}
    \IF{$\|\bg_t\|_\star=0$}
    \RETURN{$\bx_t$, optimal point}
    \ENDIF
    \STATE{$\bx_{t+1}=\bx_{t} - \frac{\alpha}{\sqrt{T}} \frac{\bg_t}{\|\bg_t\|_\star}$} 
    \ENDFOR
    \RETURN{$\bar{\bx}_T=\frac{1}{\sum_{t=1}^T \frac{1}{\|\bg_t\|_\star}}\sum_{t=1}^T \frac{\bx_t}{\|\bg_t\|_\star}$}
}
\end{algorithmic}
\end{algorithm}

Then, from very standard online learning results \citep[e.g., see][]{Orabona19}, we immediately have 
\begin{equation}
\label{eq:nogd}
\sum_{t=1}^T \frac{f(\bx_t) -f(\bu)}{\|\bg_t\|_2}
\leq \frac{\sqrt{T}\|\bu-\bx_1\|_2^2}{2\alpha}+ \frac{\alpha\sqrt{T}}{2} - \frac{\sqrt{T}}{2\alpha}\|\bx_{T+1}-\bu\|_2^2~.
\end{equation}
So, using Theorem~\ref{thm:main} and assuming the function $(L_\nu,\nu)$-H\"{o}lder smooth, we obtain
\[
f(\bar{\bx}_T) -f(\bx^\star)
\leq L_\nu \left(1+\frac{1}{\nu}\right)^\nu \left(\frac{1}{2\sqrt{T}} \left(\frac{\|\bx_1-\bx^\star\|_2^2}{\alpha}+\alpha\right)\right)^{1+\nu}~.
\]
So, we get $O(1/T)$ in the classic smooth case and $O(1/\sqrt{T})$ in the Lipschitz case, and rates in between in the general H\"{o}lder case.

We can also state a bound on the norm of the iterates: From \eqref{eq:nogd}, using the fact that $f(\bx_t)-f(\bx^\star)\geq0$, we also have that
\[
\|\bx_{T+1}-\bu\|_2^2
\leq \|\bu-\bx_1\|_2^2+\alpha^2~.
\]
Hence, the iterates are bounded. This means that the H\"{o}lder smoothness only have to hold on a ball around the initial point. So, this result generalizes the one in \citet{MishchenkoM20} because they only consider the usual notion of smoothness.

\paragraph{FTRL/DA}
With enough math, we could even prove the same thing for OGD with a time-varying learning rate. However, we can do it in a much simpler way using FTRL with linearized losses / DA.

In this case, the update is
\[
\bx_{t+1} 
= - \frac{\alpha}{\sqrt{t}} \sum_{i=1}^t \bq_i
= - \frac{\alpha}{\sqrt{t}} \sum_{i=1}^t \frac{\bg_i}{\|\bg_i\|_2}~.
\]
As before, proving that the iterates are bounded is very easy (reason as above and use Exercise 7.2 in \citet{Orabona19}). So, using the known guarantee for FTRL and reasoning as above, we obtain
\[
f(\bar{\bx}_T) -f(\bx^\star)
\leq L_\nu \left(1+\frac{1}{\nu}\right)^\nu \left(\frac{1}{\sqrt{T}}\left(\frac{\|\bx_1-\bx^\star\|_2^2}{2\alpha}+\alpha\right)\right)^{1+\nu},
\]
but this time we do not need to know $T$ ahead of time. In this case, we lost a ``2'' because $\sum_{t=1}^T \frac{1}{\sqrt{t}}\leq 2\sqrt{T}$.
Note that we could even consider the non-Euclidean case, everything is essentially the same.

\paragraph{Parameter-free version}
In the above bounds, the optimal value of $\alpha$ depends on $\|\bx_1-\bx^\star\|_2$. If you knew $\|\bx_1-\bx^\star\|_2$ (but we do not), we would obtain the term $\|\bx_1-\bx^\star\|_2$ inside the parentheses above, without the square.
Luckily, the past 11 years of research~\citep[e.g.,][and many many more!]{StreeterM12,McMahanO14,Orabona14,OrabonaP16,CutkoskyB17,OrabonaT17,FosterKMS17,CutkoskyO18,JunO19,MhammediK20,CarmonH22,ZhangCP22,ChenCO22,ZhangC22,JacobsenC23,ChzhenGS23} in parameter-free algorithms gave us an easy solution: algorithms that do not need to know $\|\bx_1-\bx^\star\|_2$ yet achieve the best rates up to (unavoidable) polylogarithmic factors.

\begin{remark}
``Parameter-free'' is a technical term like ``universal''. The definition people use is the following one: \emph{An optimization algorithm is parameter-free if it can achieve the convergence rate (in expectation or high probability) $\tilde{O}(\frac{\|\bx_1-\bx^\star\|}{\sqrt{T}})$ uniformly over all convex functions with bounded stochastic subgradients, with $T$ stochastic subgradients query and possible knowledge of the bound of the stochastic subgradients.} An analogous definition can be stated for the setting of online convex optimization as well. So, ``parameter-free'' does not mean ``without anything to tune'' nor ``without knowing anything on the function''. The name is motivated by the fact that most of the algorithms that achieve this guarantee (but not all! See, e.g., \citet{FosterKMS17}) happen to do it with methods that do not have learning rates.
\end{remark}

However, the main disadvantage of parameter-free algorithms is the need to have bounded gradients. But normalized gradients are always bounded! In particular, consider again the Euclidean case and just use the parameter-free KT algorithm~\citep{OrabonaP16} with normalized gradients:
\[
\bx_{t+1} 
= \bx_1 + \frac{-\sum_{i=1}^{t} \bq_i }{t+1}\left(d_0 - \sum_{i=1}^t \langle \bq_i, \bx_i\rangle\right)
= \bx_1 + \frac{-\sum_{i=1}^{t} \frac{\bg_i}{\|\bg_i\|_2} }{t+1}\left(d_0 - \sum_{i=1}^t \left\langle \frac{\bg_i}{\|\bg_i\|_2}, \bx_i\right\rangle\right)
\]
In case you think that this is a weird algorithm, think again: it is very difficult to beat this algorithm, because it is designed exactly for the case in which all the vectors have norm equal to 1! Again, using Theorem~\ref{thm:main} and the regret guarantee of KT, we immediately obtain
\[
f(\bar{\bx}_T) -f(\bx^\star)
\leq L_\nu \left(1+\frac{1}{\nu}\right)^\nu \left(\frac{\|\bx_1-\bx^\star\|_2}{\sqrt{T}} \sqrt{\ln \left( \frac{24T^2\|\bx_1-\bx^\star\|^2_2}{d_0^2}+1\right)}+ \frac{d_0}{T}\right)^{1+\nu}~.
\]
The advantage of this bound over the ones above is that this has almost the optimal dependency on $\|\bx_1-\bx^\star\|_2$ without the need to know anything nor to set a learning rate!

Note that this bound is not optimal in $T$ because of the $T$ term in the log. However, this is easy to fix, again with standard methods.
There are a lot of parameter-free algorithms, KT is just the simplest one. So, using the parameter-free algorithms that do not have $T$ inside the log~\citep[e.g.,][]{McMahanO14,ZhangCP22}, we can also remove the $T$ term in the log at the expense of the term $\frac{d_0}{T}$ that becomes $O(\frac{d_0}{\sqrt{T}})$ inside the parentheses.
So, the final bound would be
\[
f(\bar{\bx}_T) -f(\bx^\star)
= O\left(L_\nu \left(1+\frac{1}{\nu}\right)^\nu \left(\frac{\|\bx_1-\bx^\star\|_2}{\sqrt{T}} \sqrt{\ln \left( \frac{\|\bx_1-\bx^\star\|_2}{d_0}+1\right)}+ \frac{d_0}{\sqrt{T}}\right)^{1+\nu}\right)~.
\]
So, here we adapt to $\nu$, adapt to $\|\bx_1-\bx^\star\|_2$, adapt to $L_\nu$, and we are asymptotically optimal in $T$. Also, there are no learning rates to tune and the analysis was very simple because we just put together a known algorithm and its bound with Theorem~\ref{thm:main}.

There is also another way to obtain a similar guarantee: Just scale each gradient $\bg_t$ by $\sqrt{\sum_{i=1}^t \|\bg_i\|^2_2}$ and feed them to a parameter-free algorithm. Reasoning as in the AdaGrad-norm stepsizes, you will get the same bound; I leave the derivation as an exercise for the reader.

Note that in this case the bounded iterates are more delicate, but it can still be done using a parameter-free algorithm based on FTRL and reasoning as in \citet{OrabonaP21}.

\section{Related Work}

AdaGrad-norm stepsizes were proposed by \citet{StreeterM10} and widely used in the online learning community~\citep[e.g.,][]{OrabonaP15} before being rediscovered in the stochastic optimization, for the first time by \citet{LiO19}.

The idea of using normalized gradients to gain adaptivity goes back at least to \citet[Section 3.2.3]{Nesterov04}, where he gets adaptation to the Lipschitz constant. Reading his proof, it should be immediate to realize that the same trick works essentially for any algorithm. The majority of this report comes the underappreciated paper by \citet{Levy17}. \citet{Levy17} presents his result for a particular algorithm, but here I show that his proof applies to any online linear optimization algorithm.
As far as I know, the dependency on the geometric mean of the smoothness is new and it mirrors a similar bound I proved in the blog post\footnote{\url{https://parameterfree.com/2021/02/14/perceptron-is-not-sgd-a-better-interpretation-through-pseudogradients/}} on pseudogradients for the Perceptron.
For a paper with true universality, i.e., optimal accelerated rates in all cases, see \citet{Nesterov15b}. Also, \citet{Grimmer22} proves a more fine-grained rate for Nesterov's method on sum of functions.

Some related, but weaker and less general results appear in other papers too. \citet{MishchenkoM20} proved that one can obtain $O(1/T)$ rate in the smooth case ($\nu=1$) using gradient descent with an adaptive stepsize that tries to estimate the local smoothness. This also implies that in the non-smooth case their algorithm might fail.
\citet{MishchenkoM20} seem to have missed that \citet{Levy17} proved the stronger result of adaptation to Lipschitz and smooth case for normalized gradient descent. However, the algorithm in \citet{MishchenkoM20} can work much better in the strongly convex case because it essentially uses a learning rate that is $\approx \frac{1}{2L}$, while the algorithms based on normalized gradients will achieve a linear rate. That said, \citet{Levy17} has other results for the strongly convex case normalizing by the square of the norm of the gradient.
The guarantee on the boundedness of normalized gradient descent is new but kind of obvious and it is inspired by a similar guarantee proved in \citet{MishchenkoM20}. Moreover, the bound on the iterates of GD is well-known, see for example \citet[Corollary 2.b]{Xiao10} for the FTRL/DA case.

The adaptation to smoothness in the parameter-free case by rescaling the gradients of a parameter-free algorithm (and actually of any FTRL/DA algorithm) by $(G^2+\sum_{i=1}^{t-1} \|\bg_i\|^2_2)^{1/2+\epsilon}$ is in \citet[Lemma 26]{OrabonaP21}, where in the deterministic case $\epsilon$ can be set to 0 and $G^2$ substituted by $\|\bg_t\|^2_2$ because these changes are needed in the stochastic analysis only as explained in \citet{LiO19}.
The use of normalized gradients in parameter-free algorithms for the deterministic setting as a very easy but uninteresting alternative to avoid knowing the Lipschitz constant is folklore knowledge and it is explicitly mentioned in \citet{OrabonaP21} on page 13. Indeed, the entire problem of parameter-freeness becomes trivial and uninteresting in the deterministic setting, removing even the need to know the Lipschitz constant, as observed by Nemirovski in 2013 (his simple construction is explained in the tutorial on Parameter-free Online Optimization that we gave at ICML 2020).
\citet{KhaledMJ23} seem unaware of this folklore knowledge and rediscover it with a more convoluted proof. They also rediscover the guarantee of normalized gradient descent from \citet{Levy17}.
For the smooth deterministic case, \citet{CarmonH22} proved that a parameter-free algorithm with knowledge of the smoothness constant can achieve a rate of $O(1/T)$ paying only an additional log log factor in $\|\bx_1-\bx^\star\|$.

The first parameter-free algorithm to prove bounded iterates (with probability one) was in \citet{OrabonaP21} but without a precise bound. \citet{IvgiHC23} were the first one to design a parameter-free algorithm where $\|\bx_1-\bx_t\|_2$ is bounded by a $K \|\bx_1-\bx^\star\|_2$, where $K$ is a universal constant with high probability.

\section{Conclusion and Open Problems}
We have shown that two different reductions, rescaled gradients by AdaGrad-norm stepsizes and normalized gradients, obtain adaptive rates. There are now many open problems, for example related to the stochastic case. Indeed, the deterministic case is almost too easy to be interesting.
In fact, \citet{Levy17} has a much more complex procedure using adaptive minibatches to use normalized gradients (minus the adaptation to $\nu$) in the stochastic case.
On the issue of focusing on the deterministic setting, I feel that it is becoming the new ``let's assume a bounded domain'': Some people seem to do it not because it makes sense in their applications but just because it is easier to deal with. Indeed, it is at the same time surprising and depressing how many papers have recently appeared on this easy problem.

Another interesting direction is to obtain accelerated rates using parameter-free algorithms: it is still deterministic but there are a number of technical difficulties in using acceleration in parameter-free algorithms.

\section*{Acknowledgments}

Thanks to Benjamin Grimmer, Kfir Levy, Mingrui Liu, Nicolas Loizou, Yura Malitsky, and Aryan Mokthari for feedback and comments on a preliminary version of this paper.

\bibliographystyle{plainnat}
\bibliography{../../learning}

\end{document}